\documentclass{article}

\usepackage[numbers,sort&compress]{natbib}

\usepackage[nonatbib,preprint]{neurips_2021}

\usepackage[utf8]{inputenc} 
\usepackage[T1]{fontenc}    
\usepackage{hyperref}       
\usepackage{url}            
\usepackage{booktabs}       
\usepackage{amsfonts}       
\usepackage{nicefrac}       
\usepackage{microtype}      
\usepackage{xcolor}         
\usepackage{tikz}
\usepackage{float}
\usepackage{graphicx}
\usepackage{subcaption}
\usepackage{amsmath}
\usepackage{amsthm}
\usepackage{pifont}
\newcommand{\reffig}[1]{Figure \ref{#1}}
\newcommand{\refsec}[1]{Section \ref{#1}}
\newcommand{\refequ}[1]{Equation \ref{#1}}
\newtheorem{theorem}{Theorem}
\newtheorem{lemma}{Lemma}
\title{Chi-square Loss for Softmax: an Echo of Neural Network Structure}

\author{%
  Zeyu Wang
  \\
  School of Mechanical Engineering \& Automation\\
  Beihang University\\
  Beijing, China \\
  \texttt{wang.zeyu@buaa.edu.cn} \\
  \AND
  Meiqing Wang  \\
  School of Mechanical Engineering \& Automation\\
  Beihang University\\
  Beijing, China \\
  \texttt{wangmq@buaa.edu.cn} \\
}

\begin{document}
\bibliographystyle{unsrtnat}
\maketitle

\begin{abstract}
  Softmax working with cross-entropy is widely used in classification, which evaluates the similarity between two discrete distribution columns (predictions and true labels). Inspired by chi-square test, we designed a new loss function called chi-square loss, which is also works for Softmax. Chi-square loss has a statistical background. We proved that it is unbiased in optimization, and clarified its using conditions (its formula determines that it must work with label smoothing). In addition, we studied the sample distribution of this loss function by visualization and found that the distribution is related to the neural network structure, which is distinct compared to cross-entropy. In the past, the influence of structure was often ignored when visualizing. Chi-square loss can notice changes in neural network structure because it is very strict, and we explained the reason for this strictness. We also studied the influence of label smoothing and discussed the relationship between label smoothing and training accuracy and stability. Since the chi-square loss is very strict, the performance will degrade when dealing samples of very many classes.
\end{abstract}

\section{Introduction}
During neural network training, loss function plays a role in specifying the direction of optimization. Some classic loss functions generally have mathematical, informatics or statistical meanings, such as mean-squared-error, mean-absolute-error, cross-entropy loss, etc. In addition, there are many artificially designed loss functions, which represent the author's entry point to consider problems. For example, ranking losses often appear in metric learning, such as contrastive loss (\citet{hadsell2006dimensionality}) and triplet loss (\citet{schroff2015facenet}). In addition, circle loss by \citet{sun2020circle} is also an excellent job, which unifies cross-entropy and metric learning to a certain extent.

Our work is more based on classical statistical thinking. We are inspired by chi-square test and propose chi-square loss. It is not designed for a specific goal, so our focus is not whether chi-square loss can handle a specific problem, but hope to enrich the theoretical system of machine learning from the perspective of classical statistics. Therefore,

\begin{itemize}
\item We explained how we obtained chi-square loss from chi-square test, and proved the unbiasedness of chi-square loss in optimization. Since the chi-square loss formula has a one-hot label in the denominator, it must work with label smoothing (\refsec{chi-square loss} \& \refsec{proof and smoothing}).
\end{itemize}

We use several experiments to explain the mechanism of chi-square loss. Compared with a specially designed loss function, a plain loss function from theory is usually difficult to win in performance, but the interpretability of the loss function can help us understand the principle of neural networks and inspire subsequent research. Therefore, we want to explain the principle, characteristics and performance changes of chi-square loss from different perspectives:

\begin{itemize}
  \item In Experiment 1, we showed the characteristics of chi-square loss: strict, much stricter than cross-entropy, and this feature is especially obvious when visualizing the middle layer of neural networks. By analyzing the formula of chi-square loss, we explained the reason for this strictness, which can well explain the phenomenon in the experiment (\refsec{mnist}).
  \item In Experiment 2, we analyzed the influence of label smoothing. Label smoothing has an impact on loss, accuracy and stability, but label smoothing with better networks can get better results (\refsec{smoothing}).
  \item Chi-square loss is very strict and sensitive to the number of sample classes. In Experiment 3, we found that when the sample classes increase, the chi-square loss will have obvious performance degradation. We believe that there are two reasons for this degradation: 1. More classes make chi-square loss more likely to overfit. 2. Gradient area of the chi-square loss is relatively small, and it is not easy to find the optimizing direction. Therefore, whether the chi-square loss can play a role in a broader field requires further research (\refsec{limitations}).
\end{itemize}

\section{Chi-square test \& chi-square loss}
\label{chi-square loss}
Suppose a discrete population $X$ with $n$ classes : $P\{X=x_i\}=p_i, i=1, 2, ..., n$.

Take $k$ samples from it, and the frequency (not relative frequency) of each class is $f_i$, then define the statistic:

\begin{equation}
  C=\sum\limits_{i=1}^n\frac{(f_i-kp_i)^2}{kp_i}=\sum\limits_{i=1}^n{\frac{f_i^2}{kp_i}}-k.
\end{equation}

When $k$ is large, $C$ approximately obeys the chi-square distribution with $n-1$ degrees of freedom. Therefore, when you need to use experiments to test whether the distribution of samples is consistent with the expected distribution law, you only need to calculate the statistic $C$ based on the experimental results. The smaller the $C$, the more credible.

This work was first proposed by \citet{pearson1900x}. We transform this formula into $\sum\limits_{i=1}^n{\frac{f_i^2}{kp_i}}-k=k(\sum\limits_{i=1}^n{\frac{f_i^2 }{k^2p_i}}-1)$, which makes the variables having more clear meanings: Chi-square test is essentially to test whether the distribution of samples is consistent with the theoretically expected distribution law, and the loss function of Softmax also needs to measure the similarity of two discrete distributions, so we take the ideal distribution law $P$ as the true labels $Y=(y_1, y_2, \dots, y_n)^T$, the relative frequency ${\frac{f_i}{k}}$ as the prediction $\hat Y=(\hat y_1, \hat y_2, \dots, \hat y_n)^T$, then the statistic $C$ can be rewritten as $k(\sum\limits_{i=1}^n\frac{\hat y_i^2}{y_i}-1)$. Our goal is to minimize $\sum\limits_{i=1}^n\frac{\hat y_i^2}{y_i}-1$ (coefficient $k$ can be adjusted in learning rate). So far we have constructed a loss function:

\begin{equation}
  L=\sum\limits_{i=1}^n\frac{\hat y_i^2}{y_i}-1.
\end{equation}

This loss function has a strong statistical meaning, and we call it chi-square loss.

\section{Proof of unbiasedness \& label smoothing}
\label{proof and smoothing}
\subsection{Unbiasedness}
Although according to chi-square distribution, the smaller $\displaystyle\sum_{i=1}^n\frac{\hat y_i^2}{y_i}-1$ is, the higher the credible probability is, we still hope that even if we don't consider the condition (\textbf{approximately obey the chi-square distribution}), it is always keeping that when $\hat Y=Y$, the loss function gets the minimum value. In other words, we expect that chi-square loss is strictly unbiased in optimization process.
\begin{theorem}
  \begin{equation}
    \begin{aligned}
      &\mathop{\arg\min}_{\hat{Y}}\sum_{i=1}^n\frac{\hat y_i^2}{y_i}-1=Y,\\
      &\text{s.t. }\sum_{i=1}^n\hat y_i=\sum_{i=1}^ny_i=1, \hat y_i>0, y_i>0.
    \end{aligned}
    \end{equation}
\end{theorem}
\begin{proof}
  Using lagrange multipliers (in fact, "$=1$" is not needed):
\begin{equation*}
  L(\hat y_i,\lambda)=\sum_{i=1}^n\frac{\hat{y}_i^2}{y_i}-1+\lambda(\sum_{i=1}^n\hat y_i-\sum_{i=1}^ny_i).
\end{equation*}
Let the partial derivative of $L$ be 0:
\begin{equation*}
  \frac{\partial L}{\partial \hat y_i}=\frac{2\hat y_i}{y_i}+\lambda=0,
\end{equation*}
\begin{equation}
  2\hat y_i=-\lambda y_i,
  \label{equation 1}
\end{equation}
\begin{equation*}
  2\sum_{i=1}^n\hat y_i=-\lambda\sum_{i=1}^ny_i.
\end{equation*}
Therefore,
\begin{equation}
  \lambda=-2.
  \label{equation 2}
\end{equation}
Substitute \refequ{equation 2} into \refequ{equation 1},
\begin{equation*}
  \hat y_i=y_i, \hat Y=Y,
\end{equation*}
and easy to figure out that $\hat Y=Y$ is the global minimum point.
\end{proof}

\subsection{Label smoothing}
Because $Y$ is located in the denominator, its value cannot be $0$, but it is obvious that $Y$ is a one-hot vector and there must be "$0$" labels, therefore label smoothing is used to solve this problem. Label smoothing was first proposed by \citet{szegedy2016rethinking}. to avoid overconfidence of models and improve generalization ability. In subsequent work, label smoothing has been applied in many studies (such as \citet{real2019regularized}, \citet{chorowski2016towards}, \citet{vaswani2017attention}). \citet{muller2019does} summarized label smoothing researchs and explained principle and effect of label smoothing. For cross-entropy, label smoothing is more like a "trick", but for chi-square loss, label smoothing is necessary, otherwise the denominator will be $0$.

Usually, the function of label smoothing is to make the label softer and the calculation method is $y^{LS}_i=y_i(1-\alpha)+\alpha/n$. For example, for hard label $(1,0,0, 0,0)^T$, the soft label with $\alpha=0.1$ will be $(0.92,0.02,0.02,0.02,0.02)^T$. But for chi-square loss, label smoothing controls the intensity of punishment. The less the smoothing, the greater the punishment.

For more analysis, see \refsec{smoothing}.

\section{Experiment 1: an echo of the structure}
\label{mnist}
In experiment 1, we test chi-square loss on MNIST (\citet{Mnist}), and focus on output distribution. We visualize the penultimate layer of the model in low dimension, and illustrate the mechanism of chi-square loss.

A general dimension reduction method is to extract data from a layer of the neural network, and then map this layer to a two-dimensional or three-dimensional space in a certain way, such as t-SNE (\citet{van2008visualizing}), PCA, etc. Some researchers may concentrate on other methods, such as the aforementioned work by \citet{muller2019does}, who proposed a visualization method to study on label smoothing. All the above-mentioned methods focus on how to map, rather than the structure of neural network. If you think that the neural network itself has an ability to reduct dimensions, then you can actually set a certain layer as a low-dimensionality layer, and directly output this layer. This kind of method proposed by \citet{liu2016large} is like this: set the dimension of the penultimate layer as 2, and then plot them by classes. Compared with general methods, this method changes structure of the neural network. In Liu's work, this method is only for visualization, but in our work, it is of great significance for chi-square loss: Chi-square loss is very strict and sensitive to structure of the neural network and we use this method to illustrate the mechanism of chi-square loss.

We used the model contributed by \citet{Mnistmodel}, who provided a Pytorch version code for Liu's work. We only changed the loss function to Chi-square loss, and changed the last FC layer to a simple 2*10 layer. The learning rate used for chi-square loss is 0.0001, and for cross entropy is 0.01. The label smoothing is 0.1, and the remaining variables are the default values in the code.

As shown in \reffig{distribution}, after changing the dimension of the penultimate layer to 2, the training using chi-square loss can only get about 50\% accuracy. There are 4 categories that can be distinguished, and the remaining 6 categories are mixed together. One of the 6 categories will be considered correct, and the remaining 5 will be incorrect. On the other hand, the categories trained by cross-entropy loss can be distinguished on a 2-dimensional plane.

\begin{figure}
  \centering
  \begin{subfigure}[]{0.49\textwidth}
    \centering
    \includegraphics[]{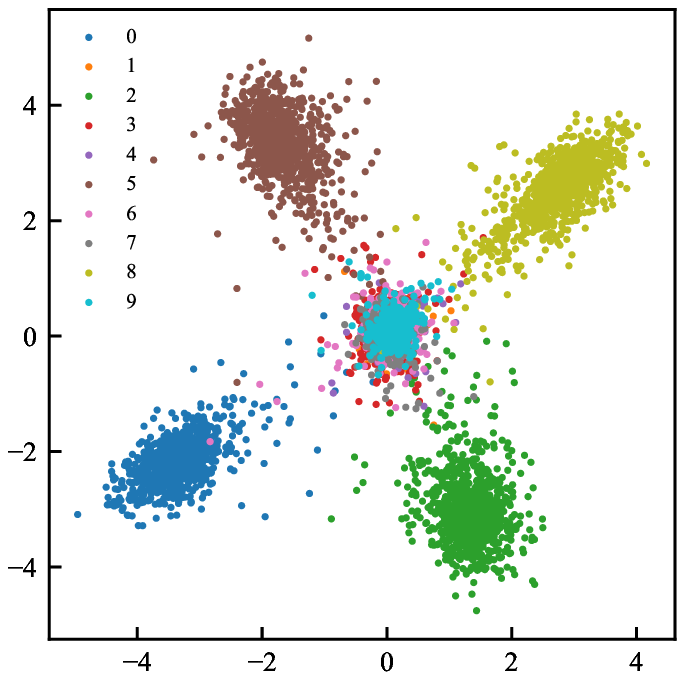}
    \caption{Chi-square loss}
  \end{subfigure}
  \begin{subfigure}[]{0.49\textwidth}
    \centering
    \includegraphics[]{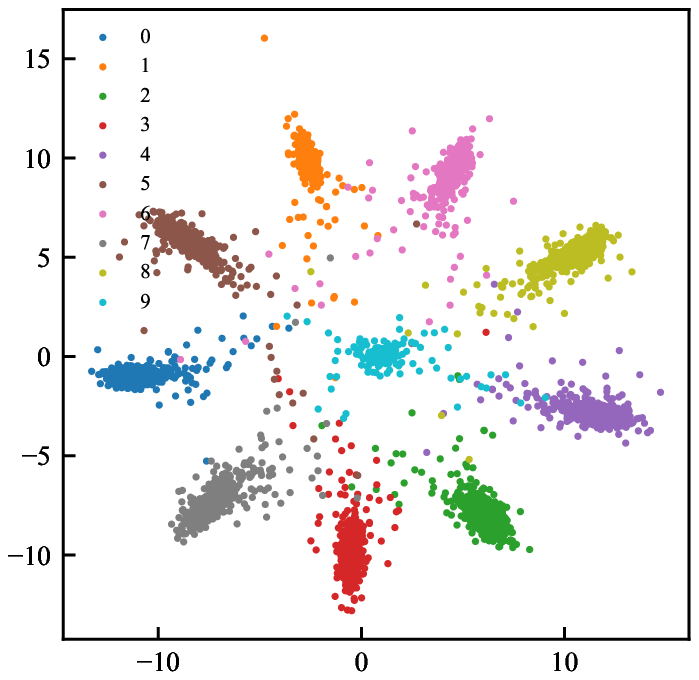}
    \caption{Cross-entropy loss}
  \end{subfigure}
  \caption{Distribution of the penultimate layer (MNIST)}
  \label{distribution}
\end{figure}

This is a phenomenon of insufficient dimensionality. As shown in \reffig{punishment}, the error source of chi-square loss is different from cross-entropy. Chi-square loss punishment is located in the denominator $y_i$. Therefore, if the label that is $0$ is recognized as $1$, it will cause loss. In a n-dimensional sample, there are $n-1$ dimensional labels will cause such a loss. Cross-entropy loss penalizes $\hat y_i$ in the logarithmic function. Recognizing a label that is 1 as 0 will cause loss. In an n-dimensional sample, only one label will cause such a loss. Therefore, chi-square loss is much stricter than cross-entropy loss.

\begin{figure}
  \centering
  \fontsize{14}{15}
  \begin{subfigure}[]{0.49\textwidth}
    \centering
    \begin{tikzpicture}[baseline,every node/.style={align=center}]
    \node at(0,0){$\displaystyle\sum\limits_{i=1}^n{\frac{\hat y^2_i}{y_i}}-1$};
    \draw[dashed](-0.048,-0.39)circle(0.28)node(end){}; 
    \node(start) at(0,-1.3){\small Error source};
    \draw[-{stealth}](start)--(-0.023,-0.73);
\end{tikzpicture}
    \caption{Chi-square loss}
  \end{subfigure}
  \begin{subfigure}[]{0.49\textwidth}
    \centering
    \begin{tikzpicture}[baseline,every node/.style={align=center}]
    \node at(0,0){$\displaystyle-\sum\limits_{i=1}^n{y_i\ln \hat y_i}$};
    \draw[dashed](1.11,0)circle(0.28)node(end){};
    \node(start) at(0.6,-1.3){\small Error source};
    \draw[-{stealth}](start)--(1.01,-0.3);  
\end{tikzpicture}
    \caption{Crossentropy loss}
  \end{subfigure}
  \caption{Different ways of punishment}
  \label{punishment}
\end{figure}

We can study this phenomenon from another perspective. Assume that for the last layer, the output $Y$ and the class label $l$ form an eleven-dimensional random variable $\begin{pmatrix} Y \\ l \end{pmatrix}$, then for cross-entropy loss, $l$ is only strongly correlated with one dimension in $Y$. For chi-square loss, 9 dimensions of $l$ and $Y$ are strongly correlated (considering the condition of $\displaystyle\sum_{i=1}^n\hat y_i=1$, $l$ and the last dimension also have some relevance). For a certain label (such as label 0), the covariance matrix of $\begin{pmatrix} Y \\ l \end{pmatrix}$ is approximately shown as \reffig{matrix} (the figure represents the last row of the covariance matrix). Red means strong correlation, and green means correlation is not important.

\begin{figure}
  \centering
  \fontsize{14}{15}
  \begin{subfigure}[]{0.49\textwidth}
    \centering
    \begin{tikzpicture}[baseline,every node/.style={align=center}]
      \foreach \x in{0,1,2,3,4,5,6,7,8,9}
      {
      \draw[fill = green!20!white,draw=green] (\x*0.5,0) rectangle (\x*0.5+0.5, 0.5);
      \node(label)at(\x*0.5+0.25, 0.75){$y_{\x}$};
      }
      \draw[fill = green!20!white,draw=green] (5,0) rectangle (5.5, 0.5);
      \node(label)at(5+0.25, 0.75){$l$};
      \node(label)at(-0.2, 0.25){$l$};
      \foreach \x in{1,2,3,4,5,6,7,8,9}
      \draw[fill = red!20!white,draw=red] (\x*0.5,0) rectangle (\x*0.5+0.5, 0.5);
      \end{tikzpicture}
    \caption{Chi-square loss}
  \end{subfigure}
  \begin{subfigure}[]{0.49\textwidth}
    \centering
    \begin{tikzpicture}[baseline,every node/.style={align=center}]
      \foreach \x in{0,1,2,3,4,5,6,7,8,9}
      {
      \draw[fill = green!20!white,draw=green] (\x*0.5,0) rectangle (\x*0.5+0.5, 0.5);
      \node(label)at(\x*0.5+0.25, 0.75){$y_{\x}$};
      }
      \draw[fill = green!20!white,draw=green] (5,0) rectangle (5.5, 0.5);
      \node(label)at(5+0.25, 0.75){$l$};
      \node(label)at(-0.2, 0.25){$l$};
      \draw[fill = red!20!white,draw=red] (0,0) rectangle (0.5, 0.5);
      \end{tikzpicture}
    \caption{Crossentropy loss}
  \end{subfigure}
  \caption{The last row of the covariance matrix}
  \label{matrix}
\end{figure}
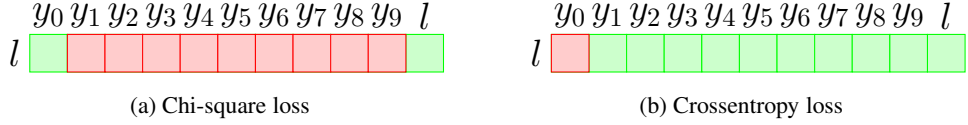

Although it is a nonlinear transformation from the penultimate layer to the output layer, the nonlinearity is not very strong. For convenience, we regard this transformation as a linear transformation $\Phi$, and assume that the output of the penultimate layer is $X$, then,
\begin{equation}
  Y=\Phi X.
\end{equation}
Adding label $l$, it can be writed as
\begin{equation}
  \begin{pmatrix} Y \\ l \end{pmatrix}=\begin{pmatrix} \Phi & 0 \\ 0 & 1 \end{pmatrix}\begin{pmatrix} X \\ l \end{pmatrix}.
\end{equation}
In \reffig{distribution}, samples of a class occupy a position, so the graphical meaning of $l$ is the position of samples.
\begin{lemma}
  \label{lemma}
  If the covariance matrix of the multidimensional random variable $A$ is $\Sigma_A$, and there is a linear transformation $\Theta$ making $B=\Theta A$, then the covariance matrix of $B$ is $\Sigma_B=\Theta \Sigma_A \Theta^T$.
\end{lemma}
We assume that the covariance matrix of $\begin{pmatrix} X \\ l \end{pmatrix}$ is $\Sigma'_X=\begin{pmatrix} \Sigma_{dd} & \sigma^T_{dl} \\ \sigma_{dl} & \Sigma_{ll} \end{pmatrix}$, $d$ means dots in \reffig{distribution}, $l$ means label. In fact, the covariance matrix of $X$ is $\Sigma_X=\Sigma_{dd}$. $\sigma$ means that it is not a complete covariance matrix. According to Lemma \ref{lemma}, The covariance matrix of $\begin{pmatrix} Y \\ l \end{pmatrix}$ is
\begin{equation}
  \begin{aligned}
    \Sigma'_Y=&\begin{pmatrix} \Phi & 0 \\ 0 & 1 \end{pmatrix}\Sigma'_X\begin{pmatrix} \Phi & 0 \\ 0 & 1 \end{pmatrix}^T\\
=&\begin{pmatrix} \Phi & 0 \\ 0 & 1 \end{pmatrix}\begin{pmatrix} \Sigma_{dd} & \sigma^T_{dl} \\ \sigma_{dl} & \Sigma_{ll} \end{pmatrix}\begin{pmatrix} \Phi & 0 \\ 0 & 1 \end{pmatrix}^T\\
=&\begin{pmatrix} \Phi\Sigma_{dd}\Phi^T & \Phi\sigma^T_{dl} \\ \sigma_{dl}\Phi^T & \Sigma_{ll} \end{pmatrix}.
  \end{aligned}
\end{equation}

$\sigma_{dl}\Phi^T$ is what shown in \reffig{matrix}. $\sigma_{dl}$ represents the relationship between dots and the label in the penultimate layer (\reffig{distribution}). If the penultimate layer has only two dimensions, then $\sigma_{dl}$ is a $1\times 2$ row vector. For cross-entropy loss, in order to make $\sigma_{dl}\Phi^T$ show as in \reffig{matrix}, only need to do is make a column in $\Phi^T$ have strong correlation with $\sigma_{dl}$. But for chi-square loss, it is necessary to make 9 columns in $\Phi^T$ have a strong correlation with $\sigma_{dl}$, and different classes of labels need to be distinguished from each other, so there will be insufficient dimensionality happening.

According to \reffig{distribution}, it is easy to find that for every additional dimension of the penultimate layer, the accuracy will increase by 20\% (each dimension can provide position of two classes). The expected relationship between dimension and accuracy is as follows:
\begin{center}
  2 for 50\%, 3 for 70\%, 4 for 90\%, and 5 or more for 100\%.
\end{center}
The results of the experiment can be consistent with the predictions. Due to limited space, no more tautology here.

We use the echo of neural network structure to describe chi-square loss, because we want to highlight the characteristic that chi-square loss is sensitive to the structure. This feature cannot be discovered through general dimensionality reduction visualization.

\section{Experiment 2: label smoothing}
In Experiment 2, we tested the CIFAR-10 dataset (\citet{CIFAR}), and we hope to demonstrate the effect of label smoothing through this experiment. We used models contributed by \citet{CIFAR10model}. We only changed loss function and set learning rate as 0.0001, and the remaining variables are the default values in the code.
\label{smoothing}
\subsection{The effect of label smoothing on loss}
For chi-square loss, the main source of loss is from zero-labels. Since $y^{LS}_i=y_i(1-\alpha)+\alpha/n$, when $y_i=0$, $y^{LS}_i=\alpha/n$. What's more, $\hat y ^2_i$ is a predicted value and has randomness, so that we assume $E(\hat y^2_i)=A$, and ignore loss from non-zero-label, in this case:
\begin{equation}
  L^{LS}=\sum_{i=1}^n\frac{\hat y_i^2}{y^{LS}_i}-1\approx\sum_{i=1,y_i=0}^n\frac{\hat y_i^2}{y^{LS}_i}-1\approx\frac{(n-1)A}{\alpha/n}-1=\frac{n(n-1)A} {\alpha}-1.
\end{equation}
Therefore, approximately,
\begin{equation}
  L^{LS}+1\propto\frac{1}{\alpha}
\end{equation}

This estimation is relatively rough. It can be seen from \reffig{smoothing on loss} that the loss and label smoothing do have an approximate inverse relationship, but not very accurate. This estimation is mainly to show that when adjusting label smoothing, the learning rate can be adjusted inversely, so that the step size of model training is relatively stable.

\begin{figure}
  \centering 
  \includegraphics[]{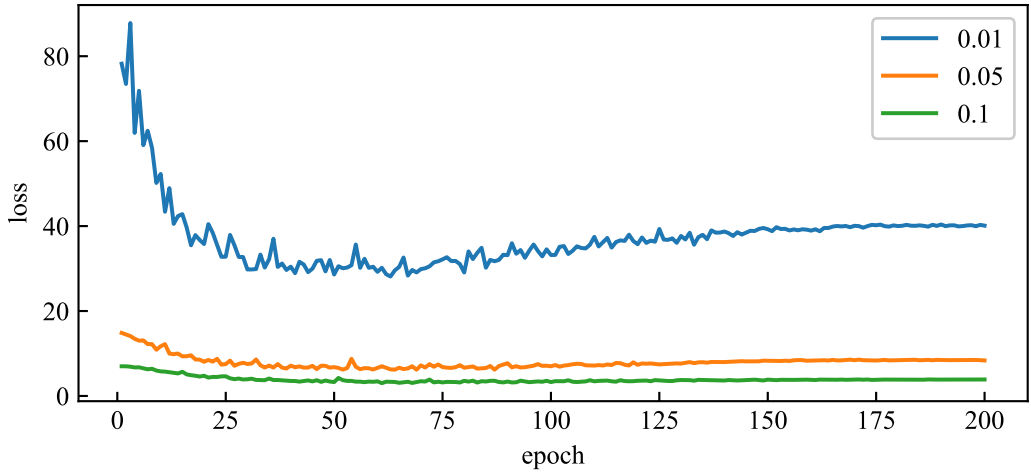}
  \caption{The loss of test set using different label smoothing, training ResNet18.}
   \label{smoothing on loss} 
 \end{figure}

 \subsection{The effect of label smoothing on accuracy and stability}
 Easy to think that the smaller the label smoothing, the stricter the loss function will be, but it will also bring in instability.

\reffig{smoothing on acc} shows the training accuracy of VGG and ResNet18 using different label smoothing. For VGG, when the label smoothing is reduced below 0.1, performance degradation occurs, while ResNet can use smaller label smoothing to obtain higher accuracy. Therefore, a neural network with a better structure can achieve better results with smaller label smoothing.
\begin{figure}
  \centering 
  \includegraphics[]{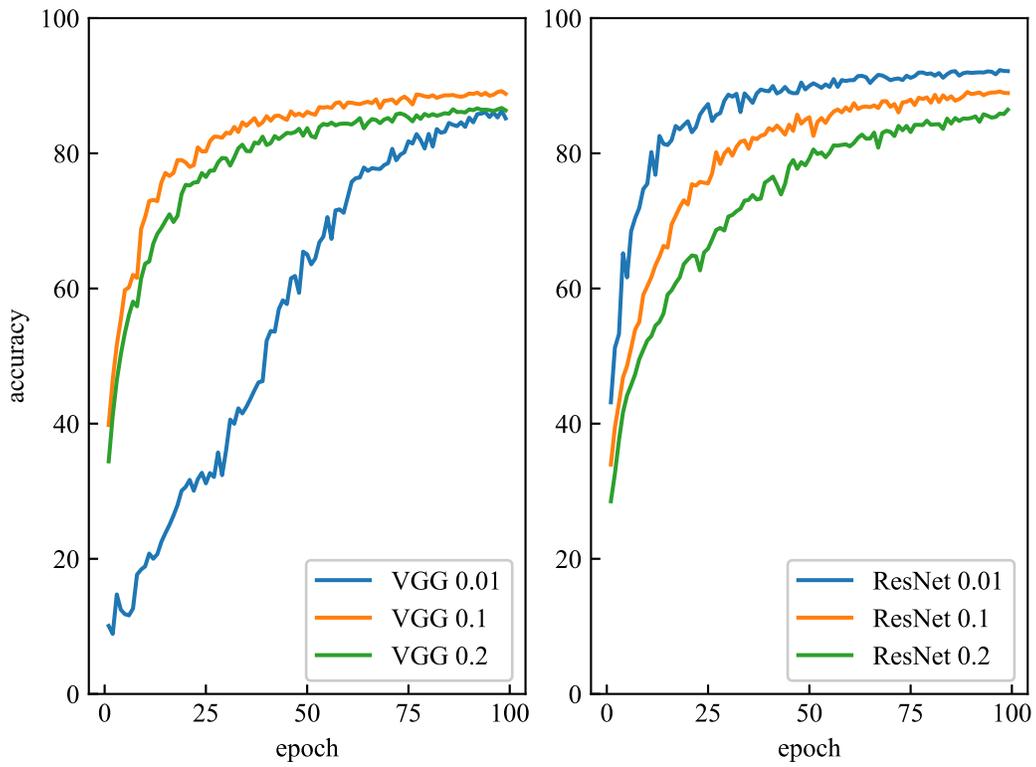}
  \caption{The accuracy of test set using different label smoothing, training VGG and ResNet18.}
   \label{smoothing on acc} 
 \end{figure}
\section{Experiment 3: limitations}
\label{limitations}
In experiment 3, we tested the CIFAR-100 dataset (\citet{CIFAR}). We use the same training parameters and models as in cifar10, only changing the dimensionality of the output layer. For $n$-dimensional samples, chi-square loss will cause $n-1$-dimension-loss, while cross-entropy only cause one-dimension-loss. It can be imagined that chi-square loss is very sensitive to the dimensionality of the sample. We extract some classes from CIFAR-100 for training, and the accuracy is shown in \reffig{decline}. We found that as the sample dimension increases, the performance of chi-square loss will degrade. We think there are two reasons for this performance degradation:
\begin{figure}
  \centering 
  \includegraphics[]{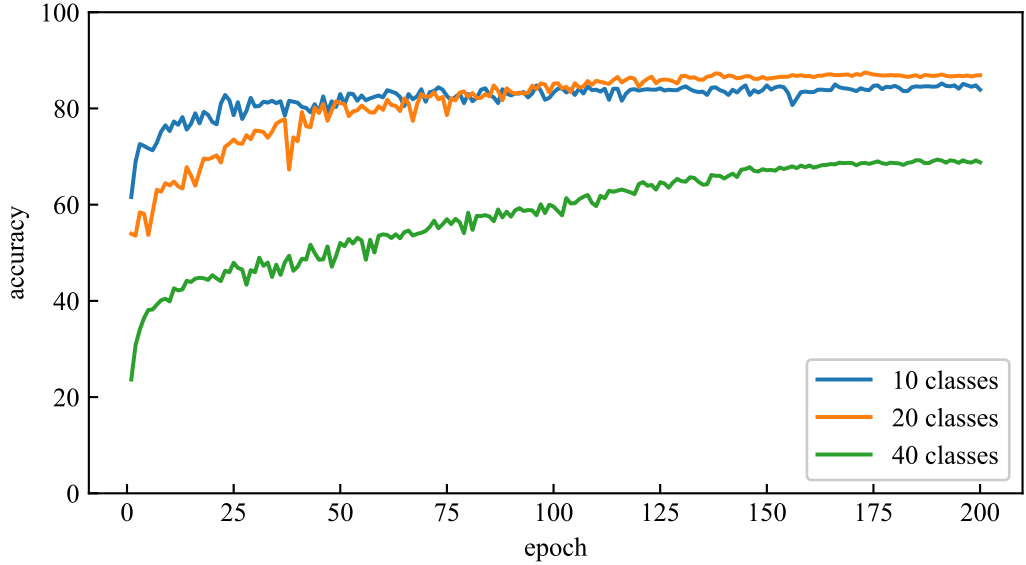}
  \caption{$n$ classes from CIFAR-100: performance degradation}
   \label{decline} 
 \end{figure}
\begin{itemize}
  \item Too many sources of loss make chi-square loss easy to overfit when training multi-class samples.
  \item We draw the error surface (before Softmax) of the output layer and found that chi-square loss only has gradient in a relatively small range, so it may be difficult to train (\reffig{error surface}).
\end{itemize}
Therefore, we think that the chi-square loss is currently not an out-of-the-box loss function.
\begin{figure}
  \centering 
  \includegraphics[]{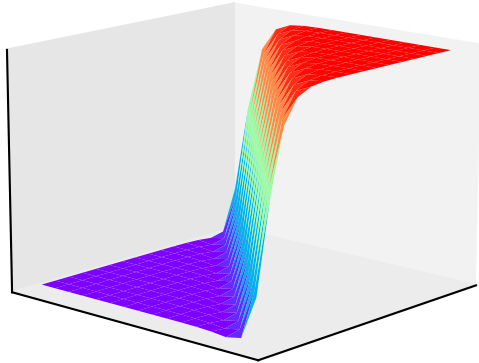}
  \caption{The error surface of chi-square loss}
   \label{error surface} 
 \end{figure}

\section{Conclusion}
We proposed chi-square loss inspired by chi-square test. We proved the unbiasedness of chi-square loss and gave the conditions of use (label smoothing). We used 3 sets of experiments to discuss the mechanism of chi-square loss, the influence of label smoothing, and its limitations. We believe that chi-square loss is not yet a completely usable loss function, but its statistical background is meaningful for enriching the theoretical system of machine learning, and its sensitivity to the neural network structure makes it like an echo of the network, which can deepen our understanding of the loss function mechanism.
\medskip

\bibliography{references}

\end{document}